\documentclass[letterpaper, 10 pt, conference]{ieeeconf} 
\IEEEoverridecommandlockouts      
\usepackage{amsmath,amssymb}
\usepackage{graphicx}
\usepackage{xcolor}
\usepackage{dsfont}
\usepackage{mathtools}
\usepackage{mathabx}
\usepackage{hyperref}  %

\usepackage{amsthm}

\usepackage{booktabs} %
\usepackage{acronym}
\usepackage{algorithm}
\usepackage{algpseudocode}

\usepackage{multirow}
\usepackage{todonotes}
\makeatletter
\let\NAT@parse\undefined
\makeatother

\newcommand{\RR}{\mathbb{R}}
\newcommand{\TT}{\mathbb{T}}
\newcommand{\PP}{\mathbb{P}}
\newcommand{\qj}{\boldsymbol{\theta}}           %
\newcommand{\dqj}{\dot{\qj}}       %
\newcommand{\tauj}{\boldsymbol{\tau}}           %
\newcommand{\obs}{\boldsymbol{o}}      %

\newcommand{\ys}{\boldsymbol{y}}       %
\newcommand{\ysMeas}{\tilde{\ys}}     %

\newcommand{\net}{\boldsymbol{\phi}}             %
\newcommand{\wht}{\boldsymbol{\psi}}             %
\newcommand{\netwht}{\net_{\wht}}     %
\newcommand{\hs}{\boldsymbol{h}}        %

\newcommand{\bel}{p}               %
\newcommand{\belPri}{\bar{\bel}}   %
\newcommand{\belPos}{\bel^+}       %
\newcommand{\belOpt}{\bel^\star}   %
\newcommand{\Idx}{\mathds{1}}      %
\newcommand{\KL}{\mathrm{KL}}      %
\newcommand{\mean}{\boldsymbol{\mu}}     %
\newcommand{\cov}{\boldsymbol{\Sigma}}   %
\newcommand{\secm}{\boldsymbol{M}}    %

\newcommand{\X}{\boldsymbol{X}}      %
\newcommand{\Xse}{\boldsymbol{\chi}}     %
\newcommand{\bxi}{\boldsymbol{\xi}} %
\newcommand{\Q}{\boldsymbol{Q}} %
\newcommand{\A}{\boldsymbol{A}} %
\newcommand{\N}{\boldsymbol{N}} %

\newcommand{\eout}{\boldsymbol{e}}      %
\newcommand{\vMeas}{\boldsymbol{\nu}}
\newcommand{\epsInt}{\boldsymbol{\epsilon}}
\newcommand{\vect}[3]{\prescript{\mathrm{#3}}{}{\boldsymbol{#1}}_{\mathrm{#2}}}
\newcommand{\dvect}[3]{\prescript{\mathrm{#3}}{}{\dot{\boldsymbol{#1}}}_{\mathrm{#2}}}
\newcommand{\rotM}[2]{\prescript{\mathrm{#2}}{\mathrm{#1}}{\boldsymbol{R}}}
\newcommand{\drotM}[2]{\prescript{\mathrm{#2}}{\mathrm{#1}}{\dot{\boldsymbol{R}}}}
\newcommand{\pos}[2]{\vect{p}{#1}{#2}}
\newcommand{\vel}[2]{\vect{v}{#1}{#2}}
\newcommand{\dpos}[2]{\dvect{p}{#1}{#2}}
\newcommand{\dvel}[2]{\dvect{v}{#1}{#2}}
\newcommand{\velMeas}[2]{\prescript{\mathrm{#2}}{}{\tilde{\boldsymbol{v}}}_{\mathrm{#1}}}
\newcommand{\SO}{\mathrm{SO}}
\newcommand{\SE}{\mathrm{SE}}
\newcommand{\aIMU}{\prescript{\mathrm{I}}{}{\tilde{\boldsymbol{a}}}}
\newcommand{\gIMU}{\prescript{\mathrm{I}}{}{\tilde{\boldsymbol{\omega}}}}
\newcommand{\aBody}{\prescript{\mathrm{I}}{}{\boldsymbol{a}}}
\newcommand{\gBody}{\prescript{\mathrm{I}}{}{\boldsymbol{\omega}}}

\newcommand{\bias}{\boldsymbol{b}}
\newcommand{\dse}{\boldsymbol{d}}

\newtheorem{assumption}{Assumption}
\newtheorem{theorem}{Theorem}
\newtheorem{remark}{Remark}

\acrodef{IMU}[IMU]{Inertial Measurement Unit}
\acrodef{OOD}[OOD]{Out-Of-Distribution}
\acrodef{EKF}[EKF]{Extended Kalman Filter}
\acrodef{IEKF}[IEKF]{Invariant Extended Kalman Filter}
\acrodef{UQ}[UQ]{Uncertainty Quantification}
\acrodef{KL}[KL]{Kullback-Leibler}
\acrodef{GML}[GML]{Gaussian Maximum Likelihood}
\acrodef{GRU}[GRU]{Gated Recurrent Unit}

\allowdisplaybreaks
\renewcommand{\baselinestretch}{0.97}
\title{\LARGE \bf Proprioceptive-only State Estimation for Legged Robots \\with Set-Coverage Measurements of Learned Dynamics
}

\author{Abhijeet M. Kulkarni$^{1}$, Ioannis Poulakakis$^{2}$ and Guoquan Huang$^{1}$
\thanks{\scriptsize$^{1}$Department of Mechanical Engineering, University of Delaware, Newark, DE, USA
        {\tt\small \{amkulk, ghuang\}@udel.edu}}%
\thanks{\scriptsize$^{2}$Robotics Institute, Athena Research Center, Marousi, Greece; School of Mechanical Engineering, National Technical University of Athens, Greece; HERON--Center of Excellence in Robotics, Athens, Greece {\tt\small poulakas@mail.ntua.gr, i.poulakakis@athenarc.gr}. }}

\begin{document}

\maketitle

\begin{abstract}
Proprioceptive-only state estimation is attractive for legged robots since it is computationally cheaper and is unaffected by perceptually degraded conditions. 
The history of joint-level measurements contains rich information that can be used to infer the dynamics of the system and subsequently produce navigational measurements. 
Recent approaches produce these estimates with learned measurement models and fuse with IMU data, under a Gaussian noise assumption.
However, this assumption can easily break down with limited training data and render the estimates inconsistent and potentially divergent. In this work, we propose a proprioceptive-only state estimation framework for legged robots that characterizes the measurement noise using set-coverage statements that do not assume any distribution. 
We develop a practical and computationally inexpensive method to use these set-coverage measurements with a Gaussian filter in a systematic way. 
We validate the approach in both simulation and two real-world quadrupedal datasets. 
Comparison with the Gaussian baselines shows that our proposed method remains consistent and is not prone to drift under real noise scenarios.

\end{abstract}

\section{Introduction and Related Work}

Quadrupedal robots have matured into reliable field platforms, demonstrating impressive robustness in challenging, unstructured environments. This capability makes them compelling candidates for missions of subterranean exploration~\cite{Miller2020RAL}, operation in natural environments~\cite{Miki2022SciRob}, and industrial inspection~\cite{Halder2023JBE}. In such domains, accurate state estimation is not merely a supporting module---it is a prerequisite for long-term autonomous operation. While exteroceptive sensors, like LiDAR, cameras, and radar, can provide highly informative measurements for accurate state estimation, they are also the most vulnerable to environmental degradation~\cite{Weisheng2026TITS}, such as dust, smoke, darkness, specularities, vegetation, and poor texture, leading to intermittent or biased perception and ultimately estimator failure. In contrast, proprioceptive sensing with \ac{IMU} and joint-level measurements is unaffected by these preceptually degraded conditions and is routinely relied upon for robust, high-rate control~\cite{Hwangbo2019SciRob}. This motivates {\em proprioceptive-only} state estimation pipelines that can maintain reliable tracking  when exteroception becomes unreliable or unavailable.

Classical proprioceptive state estimation fuses \ac{IMU} data with leg kinematics using either filtering methods~\cite{Bloesch2012RSS, Hartley2020IJRR} or factor-graph formulations~\cite{Wisth2019RAL,Noh2025arXiv}. These approaches can be effective when foot contacts are accurately detected and modeled; however, performance may deteriorate under imperfect contact. A key observation is that, during locomotion, the robot's whole-body motion and contact dynamics under the action of its low-level controller imprint rich structure onto time histories of joint angles, velocities, and torques. This structure has been exploited to extract latent states for learned dynamics models used as predictive components in planning pipelines~\cite{Roth2025RSS, Kulkarni2025Arxiv}. Analogously, a \emph{window} of proprioceptive measurements can be exploited to infer body motion. Recent learning-based approaches pursue this direction by training neural networks to extract motion-relevant patterns from proprioceptive histories and produce pseudo-measurements---such as body-frame velocity or relative displacement---that can be fused within a filter~\cite{Buchanan2021ICLR, Wasserman2024CoRL, Youm2025ICRA, Lee2025IROS}. In this view, the network functions as a virtual sensor, whose prediction error is modeled as measurement noise and fused with \ac{IMU} information in the estimation pipeline.
\begin{figure}[t]
    \centering
    \includegraphics[width=0.95\linewidth]{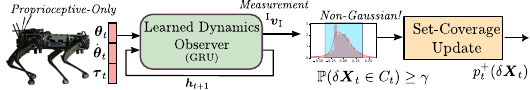}
    \caption{Proposed method uses \emph{set-coverage} measurements for update.}
    \label{fig:overview}
\end{figure}

For computational compatibility with efficient \ac{EKF} frameworks, these virtual sensors are typically modeled as corrupted by zero-mean Gaussian noise. The corresponding covariances are either fixed~\cite{Youm2025ICRA} or predicted jointly with the mean~\cite{Buchanan2021ICLR}. In practice, such networks are commonly trained with extensive data and adopt \ac{GML}~\cite{Russell2022TNNLS} to capture \emph{aleatoric} uncertainty via a predicted covariance. However, collecting large-scale, representative training data is expensive, and simulation-trained models often require careful domain randomization to mitigate sim-to-real mismatch. In the resulting limited-data regime, \emph{epistemic} uncertainty can dominate~\cite{TagasovskaNIPS2019,Gawlikowski2021AIR}, yielding prediction errors that are biased or non-Gaussian and poorly characterized by the learned aleatoric covariance~\cite{Jimenez2026Arxiv}. When such miscalibrated pseudo-measurements are fused as if they were Gaussian, the filter can become overconfident and may diverge.

A number of estimators relax the Gaussian assumption, but they tend to trade off modeling fidelity against deployability. Robust Kalman filtering variants can attenuate the impact of occasional outliers~\cite{DuranMartin2024PMLR}, yet they are less effective under persistent biases or distribution shifts. Set-membership and related bounded-error methods can handle such systematic deviations~\cite{Combastel2016ARC, Li2022RAL}, but they typically return state sets rather than point estimates, and the set representation can grow in complexity over time. Distributionally robust formulations~\cite{Wang2022TSP} provide another alternative by optimizing against an ambiguity set over noise distributions, but they require specifying that ambiguity set around a nominal likelihood and are often realized via particle filtering, which can be computationally costly. In contrast, onboard state estimation demands a fast, lightweight recursive update that still returns a point estimate while remaining reliable under non-Gaussian pseudo-measurement errors whole likelihood is not known.

This paper addresses the above failure mode by introducing a state estimation framework that replaces the untrue Gaussian noise assumption with a \emph{set-coverage} uncertainty representation for learned legged dynamics (as pseudo-measurements), see Fig.~\ref{fig:overview}. 
Specifically, we characterize pseudo-measurement error using calibrated sets that provide probabilistic coverage---i.e., with a prescribed probability, the true error lies inside the set. Such statements can be obtained via modern \ac{UQ} tools \cite{Wang2026TPAMI}, including conformal prediction~\cite{Oliveira2024JMLR} and scenario optimization~\cite{Mirasierra2021CSL}, and are distribution free. We then develop a principled method to incorporate these set-coverage constraints into a recursive filtering pipeline.
In particular, the main contributions of this paper include:
\begin{itemize}
    \item We propose a novel proprioceptive-only state estimation framework that models learned dynamics-based pseudo-measurement uncertainty via calibrated {set-coverage} statements, without the untrue Gaussian assumption.
    \item We develop a practical method to efficiently fuse coverage-constrained pseudo-measurements with the state estimate of a recursive filter.
    Our implementation is shown to be real-time 
     (e.g., $0.07\,\mathrm{ms}$ per update on a laptop CPU), suitable for high-rate onboard estimation.
    \item The proposed approach is validated extensively in Monte-Carlo simulations and on real-world quadruped experiments, showing improved robustness under non-Gaussian error behavior and preventing divergence cases observed with a Gaussian baseline (EKF) and competitive performance with proprioception baseline.
\end{itemize}

\section{Background: Legged State Estimation with Invariant Filtering}

In this section, we briefly describe the state estimation for legged robots in the \ac{IEKF} framework~\cite{Hartley2020IJRR}, which serves as the basis for our proposed method.
Specifically, we seek to continuously estimate the robot's base frame $\{I\}$ navigation state relative to a global frame $\{G\}$ using  \ac{IMU} and joint-level proprioceptive data. 
This navigation state typically comprises the orientation $\rotM{I}{G} \in \SO(3)$, velocity $\vel{I}{G} \in \RR^3$, and position $\pos{I}{G} \in \RR^3$.

\subsection{Preliminaries on $\SE_2(3)$}
The robot's navigation state is modeled as an element of the matrix Lie group $\SE_2(3)$. An element $\Xse \in \SE_2(3)$ packs the orientation, velocity, and position into a single matrix representation:
\begin{equation}
    \Xse = 
    \begin{bsmallmatrix}
        \rotM{I}{G} & \vel{I}{G} & \pos{I}{G} \\
        \boldsymbol{0}_{1 \times 3} & 1 & 0 \\
        \boldsymbol{0}_{1 \times 3} & 0 & 1 
    \end{bsmallmatrix}.
\end{equation}
The associated Lie algebra is $\mathfrak{se}_2(3)$. The hat operator $(\cdot)^\wedge : \RR^9 \rightarrow \mathfrak{se}_2(3)$ maps a vector $\mathbf{v} \in \RR^9$ to the Lie algebra, with the inverse mapping provided by the vee operator $(\cdot)^\vee : \mathfrak{se}_2(3) \rightarrow \RR^9$. The matrix exponential map $\exp(\cdot): \mathfrak{se}_2(3) \rightarrow \SE_2(3)$ maps the algebra to the group. For small perturbations $\mathbf{v}$, the exponential map admits the first-order approximation:
\begin{equation}\label{eq:lie_exp_map}
    \exp\!\left(\mathbf{v}^\wedge\right) \approx \mathbf{I} + \mathbf{v}^\wedge.
\end{equation}

A \emph{right-invariant} error between the estimated state $\widebar{\Xse}$ and the true state $\Xse$:
\begin{equation}\label{eq:right_inv_error}
    \boldsymbol{\eta} \triangleq \widebar{\Xse}\,\Xse^{-1} \in \SE_2(3).
\end{equation}
This invariant error is parameterized in the Lie algebra as 
\begin{equation}\label{eq:inv_error_parametrization}
\boldsymbol{\eta} = \exp\!\left(\bxi^\wedge\right), \quad \text{where}~ \bxi \in \RR^9
\end{equation}

\subsection{State Representation and IMU Propagation} \label{sec:state_propogation}
We augment the navigation state with \ac{IMU} biases $\bias_t\in\RR^6$:
\begin{equation*}\label{eq:setup_state_input}
\X_t=(\Xse_t,\bias_t), \quad \bias_t=\begin{bmatrix}\boldsymbol{b}_a^\top & \boldsymbol{b}_g^\top\end{bmatrix}^\top.
\end{equation*}
Using the standard \ac{IMU} model with acceleration $\aBody$ and angular rate $\gBody$ both in $\{I\}$ the measurement corrupted with bias and Gaussian noises $\boldsymbol{n}_a, \boldsymbol{n}_g$:
\begin{equation*}
\aIMU = \aBody + \boldsymbol{b}_a + \boldsymbol{n}_a,\qquad \gIMU = \gBody + \boldsymbol{b}_g + \boldsymbol{n}_g,
\end{equation*}
And the continuous-time IMU kinematics is given by:
\begin{align}\label{eq:estimator_dynamics}
\drotM{I}{G} &= \rotM{I}{G}\,[\gIMU-\boldsymbol{b}_g-\boldsymbol{n}_g]_\times, \nonumber\\
\dvel{I}{G} &= \mathbf{g}^G + \rotM{I}{G}\,(\aIMU-\boldsymbol{b}_a-\boldsymbol{n}_a), \\
\dpos{I}{G} &= \vel{I}{G},~~\dot{\boldsymbol{b}}_\mathrm{a}=\boldsymbol{n}_{ba},~~ \dot{\boldsymbol{b}}_g=\boldsymbol{n}_{bg} \nonumber
\end{align}
We collect the noise as $\boldsymbol{w}_t=[\boldsymbol{n}_a^\top~\boldsymbol{n}_g^\top~\boldsymbol{n}_{ba}^\top~\boldsymbol{n}_{bg}^\top]^\top\in\RR^{12}$, $\boldsymbol{w}_t\sim\mathcal{N}(\boldsymbol{0},\Q)$ and $\mathbf{g}^G$ is the gravity vector.

Define the augmented error state $\delta\X_t=[\bxi_t^\top~\delta\bias_t^\top]^\top\in\RR^{15}$, where $\bxi_t$ parameterizes the right-invariant error in~\eqref{eq:inv_error_parametrization} and $\delta\bias_t=\widebar{\bias}_t-\bias_t$. 
We relate the true and estimate with:
\begin{equation}\label{eq:composition_operator}
\X_t=\widebar{\X}_t\boxplus\delta\X_t \triangleq
\left(\exp\!\left(-\bxi_t^\wedge\right)\widebar{\Xse}_t,\ \widebar{\bias}_t-\delta\bias_t\right),
\end{equation}
where $\boxplus$ is the compositional operator, see \cite{Hartley2020IJRR},
and assume $\delta\X_t\sim\mathcal{N}(\boldsymbol{0},\boldsymbol{\Sigma}_t)$, i.e.,
\begin{equation}\label{eq:gaussian_composition}
\X_t \sim \widebar{\X}_t \boxplus \mathcal{N}(\boldsymbol{0}, \boldsymbol{\Sigma}_t).
\end{equation}
Linearizing the error dynamics yields
\begin{equation}\label{eq:continuous_error_dynamics}
\frac{\mathrm{d}}{\mathrm{d}t}\delta \X_t = \A_t\delta\X_t + \N_t\boldsymbol{w}_t,
\end{equation}
with $\A_t$ and $\N_t$ as in~\cite{Hartley2020IJRR}. 
Over $\Delta t$, the mean is propagated by integrating the nominal dynamics~\eqref{eq:estimator_dynamics} with zero noise and the covariance is propagated as:
\begin{equation}\label{eq:covariance_propagation}
\boldsymbol{\Sigma}_{t+\Delta t}=\boldsymbol{\Phi}\boldsymbol{\Sigma}_t\boldsymbol{\Phi}^\top+\mathbf{Q}_d.
\end{equation}
where 
$\boldsymbol{\Phi}=\exp(\A_t\Delta t)$ %
and $\mathbf{Q}_d\approx \boldsymbol{\Phi}\N_t\Q\N_t^\top\boldsymbol{\Phi}^\top\,\Delta t$.

\subsection{Invariant Measurement Output}
Following the invariant observer design, we define a right-invariant output $\ys_t$ with a known constant vector $\dse$:
\begin{align} \label{eq:meas_model}
    \ys_t &=  \Xse_t^{-1} \dse,\nonumber \\
    &=
    \begin{bsmallmatrix}        (\rotM{I}{G})^\top & -\rotM{I}{G}^\top \vel{I}{G} & -\rotM{I}{G}^\top \pos{I}{G}\\        \boldsymbol{0} & 1 & 0  \\        \boldsymbol{0} & 0 & 1  \\    \end{bsmallmatrix}
    \begin{bsmallmatrix}        \boldsymbol{0}_{3 \times 1} \\ -1 \\ 0    \end{bsmallmatrix}
    = 
    \begin{bsmallmatrix}        \vel{I}{I} \\ -1 \\ 0    \end{bsmallmatrix}.
\end{align}
where $\rotM{I}{G}^\top \vel{I}{G} = \vel{I}{I}$ is the body-frame velocity, and the last two components are constants of the homogeneous structure of the invariant output. The measurements are assumed to be corrupted by an additive noise $\vMeas_t$ such that the complete measurement model is 
\begin{equation} \label{eq:setup_meas_model}
    \ys_t =  \Xse_t^{-1} \dse + \vMeas_t
\end{equation}

In the next, we  discuss our approach to obtain the body-frame velocity $\vel{I}{I}$ and characterization of the noise $\vMeas_t$.

\section{Modeling Learned Legged Dynamics}
\label{sec:learned_measurement}

Proprioceptive signals do not directly measure navigational quantities such as the body-frame velocity $\vel{I}{I}$. However, legged locomotion induces strong dynamical couplings between joint kinematics/torques, and base motion. These couplings can be exploited to observe the system's latent state $\boldsymbol{h}$, allowing for the construction of pseudo-measurements of $\vel{I}{I}$ from short histories of proprioceptive data. In section, we learn a dynamical observer to generate these pseudo-measurements and address how to represent their prediction uncertainty for robust recursive estimation.

\subsection{Proprioceptive Observation Model}
A typical quadrupedal robot features 12 joints across its four legs, with three joints per leg (two at the hip and one at the knee). Accordingly, the joint-level proprioceptive observation at time $t$ is defined as $\obs_t \coloneqq (\qj_t, \dqj_t, \tauj_t) \in \TT^{12}\times \RR^{12} \times \RR^{12}$, representing joint angles, velocities, and torques, respectively.
Motivated by prior efforts on learning dynamical observers from short proprioceptive windows~\cite{Roth2025RSS,Kulkarni2025Arxiv}, we feed a finite proprioceptive history to a learned observer that also outputs an instantaneous pseudo-measurement of the body-frame velocity.

To model this observer, after evaluating a set of popular lightweight architectures (including MLP, RNN, LSTM, BiGRU, and TCN) on our limited training data, we employ a \ac{GRU} network as it achieved the best performance.
Let $\netwht$ denote the learned dynamical observer with parameters $\wht$ and latent state $\hs_t\in\RR^{h}$. The observer updates its latent state and outputs a pseudo-measurement of the body-frame velocity ${\velMeas{I}{I}}_t\in\RR^3$ as follows:
\begin{equation}
\label{eq:setup_rnn}
    (\hs_t,{\velMeas{I}{I}}_t) = \net_{\wht}\!\left(\obs_t,\hs_{t-\Delta t}\right)
\end{equation}
The prediction error is given by:
\begin{equation}
    \eout_t = {\velMeas{I}{I}}_t - {\vel{I}{I}}_t
    \Rightarrow \vMeas_t = \ysMeas_t - \ys_t
    = \begin{bsmallmatrix} \eout_t \\ 0 \\ 0\end{bsmallmatrix},
\end{equation}
which is also treated as the measurement noise in the right-invariant output~\eqref{eq:setup_meas_model}.

\subsection{Supervised Training}
\label{sec:training}
We train the dynamics observer network $\netwht$ using supervised trajectories with ground-truth body-frame velocities:
\begin{equation}
\label{eq:training_dataset}
\mathcal{D}_{\mathrm{train}}
=
\left\{
\left\{
\big(\obs_{k\Delta t}^{(i)}, {\vel{I}{I}}_{k\Delta t}^{(i)}\big)
\right\}_{k=0}^{T}
\right\}_{i=1}^{N_{\mathrm{train}}},
\end{equation}
where $i$ indexes trajectories, each of length $T$, and ${\vel{I}{I}}_{k\Delta t}^{(i)}$ is obtained from motion capture or a high-accuracy reference estimator. The observer latent state is set to zero at the start of each trajectory, $\hs_{0}=\boldsymbol{0}$, and parameters $\wht$ can be learned with MSE loss or \ac{GML} loss \cite{Russell2022TNNLS}, that can predict covariances. 
We use a single GRU layer, with hidden size $h=64$, train on sequences of length $T=1000$ at $\Delta t=10$\,ms, 
with the following two datasets:\footnote{Note that we here intended to choose a small network primarily to avoid overfitting to the limited data.}

\subsubsection{Vision60 dataset}\label{sec:vision_dataset_training}
This is our own dataset and the ground truth is provided by a motion-capture system. From the available sequences (Tab.~\ref{tab:ghost_dataset_results}), we train on a single trajectory to keep the supervised training protocol controlled while ensuring coverage of common locomotion regimes. Specifically, we use \textsc{3CC3CCW}, which contains repeated segments of straight walking as well as left and right turning (clockwise and counterclockwise loops), and thus exposes the observer to typical gait transitions within one continuous run. The sequence lasts $108$\,s and is $\approx 75$\,m long.

\subsubsection{Spot dataset}
This dataset~\cite{Noh2025arXiv} has the ground-truth velocity provided by a high-precision perception-based reference estimator. To mirror the Vision60 setup, we likewise train on a single representative sequence and use the remaining data for evaluation. Concretely, we train on the \textsc{Upstair} sequence, which includes sustained locomotion with direction changes under the provided reference estimates.

\subsection{Gaussianity Fails From Training to Testing}
\label{sec:gaussian_limitation}

\begin{figure}
    \centering
    \includegraphics{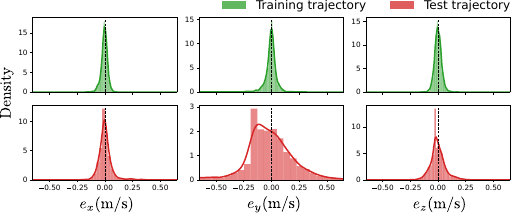}
    \caption{Prediction error $\mathbf{e}_t$ of the learned velocity predictor for Vision60 on the training trajectory (green) versus a test trajectory (red).}
    \label{fig:errors_motivation}
\end{figure}

A standard modeling assumption in learned-measurement pipelines is to treat the prediction error as a zero-mean Gaussian process:
$\mathbf{e}_t \sim \mathcal{N}(\mathbf{0}, \mathbf{R}_t)$,
which facilitates fusion with IMU data via EKF. 
This approximation is often valid when the error distribution is unbiased and unimodal, to achieve this performance typically a large, representative training datasets is required. 
As shown in Fig.~\ref{fig:errors_motivation}, the Vision60 network's error on the training data closely approximates a Gaussian distribution, where aleatoric uncertainty can be effectively captured by GML-based covariance prediction.

However, in realistic field deployments where training data is limited, epistemic uncertainty becomes dominant when the model encounters out-of-training distribution data~\cite{TagasovskaNIPS2019, Gawlikowski2021AIR}. 
In such cases, the induced error may exhibit varying multi-modality, skewness, or heavy tails—characteristics not captured by a fixed parametric noise model. 
Fig.~\ref{fig:errors_motivation} illustrates this phenomenon: when the network trained on a single sequence (Section~\ref{sec:vision_dataset_training}) is evaluated on a test trajectory, the prediction errors {\em deviate significantly from Gaussianity}, especially in the forward $y$-direction of motion.

\subsection{Error-State Coverage Statements}
\label{sec:uq_coverage}

To overcome the above limitations of Gaussian noise models, we characterize the learned predictor errors using \emph{set-coverage statements} \cite[Section 9.3.2]{casella2002statistical}, which constrain only the {\em probability mass} assigned to a calibrated set and are therefore agnostic to the  unknown (or hard-to-model) true {\em probability distribution} of the errors. 
Concretely, for the velocity prediction error $\eout_t$, we assume the following bound $\epsInt\in\RR^3_{\ge 0}$ such that
\begin{equation}
\label{eq:coverage_statement}
    \PP\!\left(|\eout_t| \leq \epsInt\right) \geq \gamma,
\end{equation}
where $|\cdot|$ denotes element-wise absolute value and $\gamma\in(0,1)$ is the desired confidence level (probability). Fig.~\ref{fig:coverage_validity} illustrates an example with $\gamma=0.85$, where the same coverage set remains compatible with both in-distribution (training) and shifted (test) errors.

Note that the coverage bounds~\eqref{eq:coverage_statement} can be obtained using distribution-free \ac{UQ} procedures such as conformal prediction \cite{Shafer2008JMLR} or scenario-optimization-based calibration \cite{Mirasierra2021CSL}. 
Because these coverage statements can be computed \emph{post-hoc}, they apply to any frozen predictor~\cite{Shafer2008JMLR} and avoid costly retraining when the deployment environment shifts, while capturing both aleatoric and epistemic uncertainty~\cite{Sale2025PMLR}. 
If the training data poorly represent calibration/deployment conditions, the resulting $\epsInt$ may become conservative. 
We will detail the specific calibration procedure used in our experiments in Section~\ref{sec:calibration_and_coverage_experiments}.

\begin{remark}\label{rem:weakness_of_coverage}
    A coverage statement is weaker than explicit likelihood. It constrains only the probability mass assigned to a set and leaves the distribution inside and outside that set unspecified. This makes the framework flexible enough to accommodate non-Gaussian, multimodal, and heavy-tailed errors distributions. This is a core advantage and enables the use of coverage guarantees without requiring restrictive modeling assumptions on the measurement noise.
\end{remark}

\begin{figure}
    \centering
    \includegraphics[width=\columnwidth]{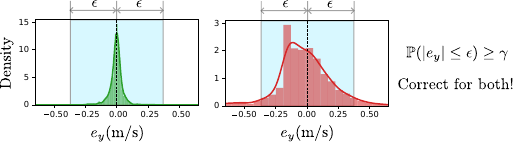}
    \caption{A single calibrated coverage statement can remain valid for both training (green) and test (red) error distributions.}
    \label{fig:coverage_validity}
\end{figure}

To the best of our knowledge, we are the {\em first} to utilize this flexible set-coverage statement~\eqref{eq:coverage_statement} to model learned legged dynamics, rather than using an ad-hoc Gaussian model.
Specifically, we express the calibrated velocity-error coverage statement directly in the augmented error state $\delta\X_t$~\eqref{eq:composition_operator}. 
Using the measurement model \eqref{eq:setup_meas_model} together with $\vMeas_t=[\eout_t^\top~0~0]^\top$. We define the projection 
$\Pi \coloneqq \begin{bsmallmatrix}\boldsymbol{I}_{3\times 3} & \boldsymbol{0}_{3\times 2}\end{bsmallmatrix}$, 
so that the coverage set $|\eout_t|\le \epsInt$ is equivalently
\begin{equation}
    \left|\Pi\!\left(\Xse_t^{-1}\dse-\ysMeas_t\right)\right|\le \epsInt.
\end{equation}
With the right-invariant error \eqref{eq:right_inv_error}, substituting $\Xse_t^{-1}=\widebar{\Xse}_t^{-1}\boldsymbol{\eta}_t$ in above and using first-order approximation \eqref{eq:lie_exp_map}, we have
\begin{equation}
    \Pi\!\left(\Xse_t^{-1}\dse-\ysMeas_t\right)
    \approx
    \Pi\!\left(\widebar{\Xse}_t^{-1}\dse-\ysMeas_t\right)
    + \boldsymbol{H}\,\delta\X_t,
\end{equation}
where $\boldsymbol{H}\in\RR^{3\times 15}$ is the linear map
$\boldsymbol{H}\delta\X_t=\Pi\,\widebar{\Xse}_t^{-1}\bxi_t^\wedge\dse$ (the bias components do not enter). Hence the coverage constraint induces the error-state feasible set:
\begin{align}
    C_t &= \left\{ \delta \X_t \in \RR^{15} : \boldsymbol{l}_t \;\leq\; \boldsymbol{H} \delta \X_t \;\leq\; \boldsymbol{u}_t \right\}
    \label{eq:setup_feasible_set_lie_linear_augmented}\\
    \boldsymbol{l}_t &= -\Pi \left( \widebar{\Xse}_t^{-1} \dse - \ysMeas_t \right) - \epsInt\\
    \boldsymbol{u}_t &= -\Pi \left( \widebar{\Xse}_t^{-1} \dse - \ysMeas_t \right) + \epsInt
\end{align}
Clearly, \eqref{eq:coverage_statement} is equivalent to the error-state coverage statement (see \cite[Thm. 9.2.2]{casella2002statistical}):
\begin{equation}\label{eq:coverage_constraint_error_state}
    \PP(\delta \X_t \in C_t) \geq \gamma.
\end{equation}

In the next, we will discuss how the error-state coverage statement is used to update the current state estimate. 
Once the posterior $\delta \X_t^+ \sim \mathcal{N}(\mean_t^+, \boldsymbol{P}_t^+)$ satisfying \eqref{eq:coverage_constraint_error_state} is obtained, the on-manifold update is $\X_t^+ = \widebar{\X}_t \boxplus \delta \X_t^+$.

\section{Our Proprioceptive-only State Estimator} \label{sec:proposed_update_method}

In this section, we present how to rigorously incorporate the set-coverage statement~\eqref{eq:coverage_constraint_error_state} into 
our proposed proprioceptive-only state estimator in the \ac{IEKF} framework.
As the \ac{IMU} propagation is standard as in Section~\ref{sec:state_propogation}, 
in the following we focus on the coverage measurement update and its practical efficient implementation.

\subsection{Coverage-Constrained Update}

After the \ac{IEKF} propagation with IMU (see Section~\ref{sec:state_propogation}), 
we have the following prior Gaussian  estimate  at time $t$:
\begin{equation}
    \belPri_t(\delta\X) \coloneqq \mathcal{N}(\delta\X;\widebar{\mean}_t,\widebar{\cov}_t),
\end{equation}
with $\widebar{\mean}_t=0$ (kept for the generality).
We seek to find the posterior distribution $\belPos(\delta\X)$ that:
(i) assigns at least $\gamma$ probability mass to $C_t$, and
(ii) deviates minimally from $\belPri_t$ in \ac{KL} divergence, as in minimum cross-entropy
\cite{Shore1981TIT} and posterior regularization \cite{Ganchev2007NIPS} approaches: 
\begin{align}
\belOpt_t \coloneqq \arg\min_{\bel}\quad &
\KL\!\left(\bel\,\|\,\belPri_t\right)   \label{eq:kl_projection_error_state}\\
\text{s.t.}\quad &
\int_{C_t} \bel(\delta \X)\,\mathrm{d}(\delta \X) \ge \gamma,\nonumber\\
&\int \bel(\delta \X)\,\mathrm{d}(\delta \X)=1, &\bel(\delta \X)\ge 0 \nonumber
\end{align}
Note that the objective $\KL(\bel\|\belPri_t)$ is strictly convex in $\bel$, and the constraints are linear in $\bel$,
hence \eqref{eq:kl_projection_error_state} is a convex optimization problem \cite{Ganchev2007NIPS} and assumes a unique minimizer. 
The following result shows that enforcing the coverage constraint admits a closed-form  update solution:

\begin{theorem}[KL-minimal posterior with a set-mass constraint]
\label{thm:kl_projection_error_state}
Assume $\belPri_t(\delta \X)>0$ almost everywhere on its support. Let $C_t$ be measurable, $\gamma\in(0,1)$, and let
\begin{equation}\label{eq:pi_def_error_state}
\pi_t \coloneqq \int_{C_t} \belPri_t(\delta \X)\,\mathrm{d}(\delta \X).
\end{equation}
If $\pi_t\ge \gamma$, then the unique minimizer of \eqref{eq:kl_projection_error_state} is $\belOpt_t=\belPri_t$.
If $\pi_t<\gamma$ and $0<\pi_t<1$, then the unique minimizer is
\begin{equation}
\label{eq:posterior_piecewise}
\belOpt_t
=
\frac{\gamma}{\pi_t}\belPri_t\,\Idx_{C_t}(\delta \X)
+
\frac{1-\gamma}{1-\pi_t}\belPri_t\,\Idx_{(C_t)^c}(\delta \X).
\end{equation}
\end{theorem}
\begin{proof}
See Appendix~\ref{app:proof}.
\end{proof}

This result implies a simple update rule. 
If the prior already meets the coverage requirement
($\pi_t\ge\gamma$), the constraint is inactive and no modification is required: 
$\belPos_t=\belPri_t$, hence $\mean_t^+=\widebar{\mean}_t$ and $\cov_t^+=\widebar{\cov}_t$.
If instead $\pi_t<\gamma$, the KL-minimal projection rescales the prior inside and outside $C_t$ so that
$\PP_{\belOpt_t}(\delta\X\in C_t)=\gamma$. Since the resulting $\belOpt_t$ is generally non-Gaussian, we
restore a Gaussian representation by moment matching for recursive Guassian estimation.\footnote{Moment matching is equivalent to reverse-direction KL minimization $q^\star=\arg\min_{q\in\mathcal{G}}\KL(q\|\belOpt_t)$, where $\mathcal{G}$ is the family of Gaussians.}

\subsection{Gaussian Moment Matching}
Our goal is to construct a Gaussian posterior $\belPos_t$ whose first two moments match those of the optimal posterior $\belOpt_t$.
To this end, define the prior second moment as $
\secm_t \coloneqq \widebar{\cov}_t+\widebar{\mean}_t\widebar{\mean}_t^\top .
$ Then the truncated moments of the prior on $C_t$ are:
\begin{align}\label{eq:trunc_moments_error_state}
    \mean_{C_t} &\coloneqq \frac{1}{\pi_t}\int_{C_t} \delta \X\, \belPri_t(\delta \X)\, \mathrm{d}(\delta \X),\nonumber\\
    \secm_{C_t} &\coloneqq \frac{1}{\pi_t}\int_{C_t} \delta \X\delta \X^\top\, \belPri_t(\delta \X)\, \mathrm{d}(\delta \X).
\end{align}
The corresponding moments on the complement follow from the law of total expectation:
\begin{equation}
\label{eq:complement_moments}
    \mean_{C_t^c}
    = \frac{\widebar{\mean}_t-\pi_t\mean_{C_t}}{1-\pi_t},\qquad
    \secm_{C_t^c}
    = \frac{\secm_t-\pi_t\secm_{C_t}}{1-\pi_t}.
\end{equation}
Finally, the moment-matched Gaussian posterior $\belPos_t(\delta \X)=\mathcal{N}(\delta\X;\mean_t^+,\cov_t^+)$ is obtained as
\begin{align}\label{eq:moment_matching}
\mean_t^+ &= \gamma \mean_{C_t} + (1-\gamma)\mean_{C_t^c},\\
\secm_t^+ &= \gamma \secm_{C_t} + (1-\gamma)\secm_{C_t^c},\\
\cov_t^+ &= \secm_t^+ - \mean_t^+(\mean_t^+)^\top.
\end{align}

\begin{figure}
    \centering
    \includegraphics{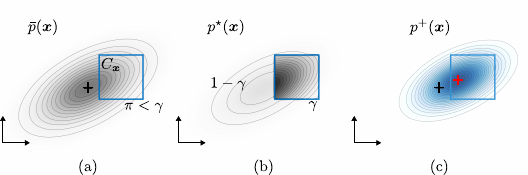}
    \caption{Illustration of the KL-minimal belief update with a coverage constraint. (a) Prior with $\pi_t<\gamma$.
    (b) Exact KL-optimal posterior: piecewise rescaling of the prior inside/outside $C_t$.
    (c) Gaussian approximation obtained by moment matching.}
    \vspace{-8pt}
    \label{fig:update_illustration}
\end{figure}
Fig.~\ref{fig:update_illustration} illustrates the update process from obtaining optimal posterior to its moment-matched Gaussian. 
The main computational cost is evaluating
$\pi_t$ in \eqref{eq:pi_def_error_state} and the truncated moments in \eqref{eq:trunc_moments_error_state},
which require integrating a multivariate Gaussian over the set $C_t$ in $\RR^{15}$ (see \eqref{eq:setup_feasible_set_lie_linear_augmented}).
Next we exploit the structure of $C_t$ to reduce the integration dimension and then use numerical methods~\cite{Genz2016SIP} to compute these integrals.

\subsection{Reduced-Dimensional Implementation} %
\label{sec:reduced_dim_update}

Note that for \eqref{eq:setup_feasible_set_lie_linear_augmented}, the constraint depends on $\delta\X_t$ only through
$\boldsymbol{z}_t \coloneqq \boldsymbol{H}\delta\X_t \in \RR^{3}$, with component-wise bounds
$\boldsymbol{l}_t \le \boldsymbol{z}_t \le \boldsymbol{u}_t$.
Thus we can perform the expensive probability and truncated-moment computations in the output space $\RR^{3}$.
Under the Gaussian prior $\delta\X_t\sim\mathcal{N}(\widebar{\mean}_t,\widebar{\cov}_t)$, the projected prior on
$\boldsymbol{z}_t$ is given by:
\begin{equation}\label{eq:output_space_prior}
\boldsymbol{z}_t \sim \mathcal{N}(\widebar{\mean}_{z,t}, \widebar{\cov}_{z,t})
= \mathcal{N}\!\Bigl(\boldsymbol{H}\widebar{\mean}_t,\; \boldsymbol{H}\widebar{\cov}_t\boldsymbol{H}^\top\Bigr).
\end{equation}
The coverage set becomes a axis-aligned box:
\begin{equation}
C_t^{z} \coloneqq \{\boldsymbol{z}\in\RR^{3}\mid \boldsymbol{l}_t \le \boldsymbol{z} \le \boldsymbol{u}_t\}.
\end{equation}
We apply Theorem~\ref{thm:kl_projection_error_state} and the moment-matching update \eqref{eq:moment_matching}
directly in $\boldsymbol{z}$-space to obtain an updated marginal
$\boldsymbol{z}_t \sim \mathcal{N}(\mean_{z,t}^+,\cov_{z,t}^+)$, where $\pi_t=\PP(\boldsymbol{z}_t\in C_t^z)$ on the prior.

To recover a Gaussian distribution over $\delta\X_t\in\RR^{15}$, we keep the \emph{prior conditional}
$p(\delta\X_t\mid \boldsymbol{z}_t)$ and replace only the marginal over $\boldsymbol{z}_t$.
This preserves the prior cross-correlations between constrained and unconstrained components while updating
uncertainty only in the directions informed by the coverage constraint.
Then the lifted Gaussian posterior has moments:
\begin{align}
\mean_t^+ &= \widebar{\mean}_t + \boldsymbol{K}_t(\mean_{z,t}^+ - \widebar{\mean}_{z,t}), \label{eq:lift_mean}\\
\cov_t^+  &= \widebar{\cov}_t + \boldsymbol{K}_t(\cov_{z,t}^+ - \boldsymbol{S}_t)\boldsymbol{K}_t^\top.
\label{eq:lift_cov}
\end{align}
where 
$\boldsymbol{S}_t = \boldsymbol{H}\widebar{\cov}_t\boldsymbol{H}^\top$ and $\boldsymbol{K}_t = \widebar{\cov}_t\boldsymbol{H}^\top \boldsymbol{S}_t^{-1}$.
Algorithm~\ref{alg:coverage_update_short} summarizes the main steps of this implementation.

\begin{remark}\label{rem:interpretation_of_lift}
Equation~\eqref{eq:lift_cov} admits a useful interpretation. If $\cov_{z,t}^+=\boldsymbol{0}$, then
\eqref{eq:lift_cov} reduces to the standard Kalman covariance update corresponding to an exact (noise-free)
observation of $\boldsymbol{z}_t$. For $\cov_{z,t}^+\succcurlyeq \boldsymbol{0}$, the term
$\boldsymbol{K}_t\cov_{z,t}^+\boldsymbol{K}_t^\top$ re-injects residual uncertainty along the output
directions, reflecting that the set-coverage statement specifies a set-probability condition rather than a
noise-free measurement.
\end{remark}

\begin{algorithm}[t]
\caption{Coverage-Constrained Measurement Update}
\label{alg:coverage_update_short}
\begin{algorithmic}[1]
\Require Prior state $\widebar{\X}_t$, prior error moments $(\widebar{\mean}_t,\widebar{\cov}_t)$, pseudo-measurement $\ysMeas_t$, bounds $\epsInt$, confidence $\gamma$
\Ensure Updated state $\X_t^+$, updated error covariance $\cov_t^+$
\State Form $\boldsymbol{l}_t,\boldsymbol{u}_t$, $\boldsymbol{H}$, and $C_t$ via \eqref{eq:setup_feasible_set_lie_linear_augmented}.
\State Project prior to constraint space using \eqref{eq:output_space_prior}.
\State Compute $\pi_t=\PP(\boldsymbol{z}_t\in C_t^{z})$ and truncated moments in $\boldsymbol{z}$-space with \cite{Genz2016SIP}.
\If{$\pi_t \ge \gamma$}
    \State $\mean_t^+ \gets \widebar{\mean}_t,\quad \cov_t^+ \gets \widebar{\cov}_t,\quad \X_t^+ \gets \widebar{\X}_t$
    \State \Return $\X_t^+,\cov_t^+$
\EndIf
\State Moment match in $\boldsymbol{z}$-space to obtain $(\mean_{z,t}^+,\cov_{z,t}^+)$ via \eqref{eq:moment_matching}.
\State Lift to full space using \eqref{eq:lift_mean}--\eqref{eq:lift_cov} to get $(\mean_t^+,\cov_t^+)$.
\State On-manifold update: $\X_t^+ \gets \widebar{\X}_t \boxplus \mean_t^+$.
\State \Return $\X_t^+,\cov_t^+$
\end{algorithmic}
\end{algorithm}

\section{Monte-Carlo Simulations}
We validate the proposed set-coverage measurement update in Monte Carlo simulations. The objectives are to:
(i) characterize the runtime--accuracy trade-off of estimating the truncated probability mass and moments required by the update, and
(ii) compare filter behavior against a standard \ac{IEKF} \cite{Hartley2020IJRR} correction under both correctly specified Gaussian noise and deliberately misspecified non-Gaussian noise.

Theorem~\ref{thm:kl_projection_error_state} and the moment matching in \eqref{eq:moment_matching} require prior inset probability mass \eqref{eq:pi_def_error_state} $\pi_t$ and truncated moments \eqref{eq:trunc_moments_error_state} under the prior $\boldsymbol{z}_t\sim\mathcal{N}(\widebar{\mean}_{z,t},\widebar{\cov}_{z,t})$. We approximate these integrals with a randomized (quasi-)Monte Carlo estimator for Gaussian box probabilities and moments (cf.\ \cite{Genz2016SIP}). We compare accuracy with absolute error in $\pi_t$, $\ell_2$ error of the truncated mean, and Frobenius norm of the truncated covariance, each relative to a $10^7$-sample reference. 
Tab.~\ref{tab:mc_timing_error} shows the resulting runtime--accuracy trade-off on a Laptop i9-13950HX CPU. We use $N=1000$ samples in all subsequent experiments, which yields sub-millisecond cost ($\approx0.07$\,ms/update) with sufficiently small errors.

\begin{table}[t]
\centering
\caption{Runtime--accuracy for truncated-moment computations. %
}
\label{tab:mc_timing_error}
\vspace{-8pt}
\setlength{\tabcolsep}{4.5pt}
\renewcommand{\arraystretch}{1.12}
\begin{tabular}{rcccc}
\hline
\textbf{Samples} & \textbf{Timing (ms)} & \textbf{Prob. Error} & \textbf{Mean Error} & \textbf{Cov. Error} \\
\hline
$100$        & 0.0125 & 3.12e-4 & 1.03e-1   & 2.32e-1   \\
$500$        & 0.0372 & 8.07e-5 & 2.10e-2  & 6.72e-2  \\
$\mathbf{1000}$      & 0.0664 & 2.00e-5 & 1.48e-2  & 3.14e-2  \\
$5000$      & 0.302  & 5.33e-6 & 3.46e-3 & 6.46e-3 \\
$10000$     & 0.532  & 1.06e-6 & 2.06e-3 & 3.20e-3 \\
\hline
(Ref.) $1\mathrm{e}7$ & 818.23 & 0 & 0 & 0 \\
\hline
\end{tabular}
\end{table}

To set up the comparison, we simulate a rigid body following a quadrupedal base-motion trajectory and generate ideal IMU and body-frame velocity outputs consistent with \eqref{eq:estimator_dynamics}--\eqref{eq:setup_meas_model}. IMU signals are corrupted with known fixed biases and additive white noise. The body-frame velocity pseudo-measurement is corrupted by an additive error process $\eout_t$. 

We compare a standard \ac{IEKF} correction that assumes $\eout_t$ is zero-mean Gaussian with covariance $\boldsymbol{R}$, against the proposed set-coverage update in $\boldsymbol{z}$-space (Section~\ref{sec:reduced_dim_update}). The set-coverage update uses $N=1000$ samples to estimate $\pi_t$ and truncated moments, followed by the lift to the error-state space. For this update, the bounds $(\boldsymbol{l}_t,\boldsymbol{u}_t)$ are formed from the calibrated elementwise radius $\epsInt$ as in \eqref{eq:setup_feasible_set_lie_linear_augmented}, set from the $\gamma$-quantiles of the noise used to generate measurements.

We evaluate these methods under two specific noise regimes. In the unbiased Gaussian regime, $\eout_t \sim \mathcal{N}(\boldsymbol{0},\boldsymbol{R}_\star)$ and the \ac{IEKF} uses $\boldsymbol{R}=\boldsymbol{R}_\star$. This isolates the performance cost of replacing a full likelihood with a weaker set-coverage statement (Remark~\ref{rem:weakness_of_coverage}). In the misspecified non-Gaussian regime, we generate $\eout_t$ from a Gaussian mixture with nonzero component means. For each Monte Carlo trial, we draw one mixture component and keep it fixed over the entire trajectory. The \ac{IEKF} uses a single Gaussian covariance $\boldsymbol{R}$ fitted to the mixture, while the coverage bounds are set from outer quantiles. This produces persistent, trial-dependent bias relative to the \ac{IEKF}'s assumed zero-mean Gaussian model.

Tab.~\ref{tab:unbiased_biased_combined} reports position RMSE and position NEES ($\mathrm{NEES}_{\mathrm{pos}}=3$ for 3D position, ideally). In the well-specified Gaussian regime, the IEKF achieves the best accuracy and near-nominal consistency, whereas the set-coverage update has higher RMSE because it only exploits mass in the set rather than the full Gaussian shape. However, in the misspecified regime, the \ac{IEKF} becomes severely inconsistent with a large NEES due to overconfident corrections under an incorrect likelihood model. In contrast, the set-coverage update maintains near-nominal consistency across $\gamma$ while remaining competitive in RMSE. Operationally, this benefit comes from re-injecting uncertainty along the constrained directions (Remark~\ref{rem:interpretation_of_lift}). Ultimately, the set-coverage update is most valuable when pseudo-measurement errors are misspecified, such as being non-Gaussian or biased over a trajectory, where it improves consistency by mitigating overconfident corrections.

\begin{remark}
The moment-matched Gaussian approximation does not, in general, preserve the exact set-mass constraint, so the post-approximation probability mass in the feasible set need not equal $\gamma$.
Empirically, we observe that the update reliably increases the feasible-set mass toward the target (typically $\pi_t < \pi_t^+ \leq \gamma$).
While one could iterate the projection/moment-matching to approach $\gamma$ more tightly, we find that at high update rates a single update per timestep is sufficient in practice. A formal analysis of this behavior is left for the future work.
\end{remark}

\begin{table}[t]
\caption{Unbiased/Biased noise Monte Carlo comparison.}
\label{tab:unbiased_biased_combined}
\vspace{-8pt}
\centering
\resizebox{\columnwidth}{!}{%
\begin{tabular}{c c c}
\hline
$\gamma$ & RMSE (m) (unbiased/biased) & NEES$_{\mathrm{pos}}$ (unbiased/biased) \\
\hline
\multicolumn{3}{c}{\textbf{EKF update}} \\
\hline
-- & $0.6133 \pm 0.2546\,/\,2.3298 \pm 1.4016$ & $3.1139 \pm 2.5023\,/\,99.8873 \pm 121.1403$ \\
\hline
\multicolumn{3}{c}{\textbf{Set-coverage update}} \\
\hline
0.70 & $0.9063 \pm 0.3111\,/\,2.9453 \pm 1.0126$ & $2.5220 \pm 1.1528\,/\,1.8621 \pm 1.3769$ \\
0.75 & $0.8500 \pm 0.2590\,/\,2.8397 \pm 1.2102$ & $3.5934 \pm 2.4633\,/\,3.0924 \pm 3.6713$ \\
0.80 & $0.9402 \pm 0.2702\,/\,2.6030 \pm 1.3559$ & $4.8568 \pm 2.4530\,/\,1.9772 \pm 2.6391$ \\
0.85 & $0.8600 \pm 0.3257\,/\,3.0532 \pm 1.3384$ & $4.4314 \pm 4.3258\,/\,3.7381 \pm 5.5225$ \\
0.90 & $0.9450 \pm 0.2910\,/\,2.6683 \pm 1.1898$ & $4.6952 \pm 2.8569\,/\,4.0185 \pm 5.2741$ \\
0.95 & $1.0468 \pm 0.3559\,/\,3.3612 \pm 1.4733$ & $3.8479 \pm 2.4304\,/\,3.7676 \pm 4.6494$ \\
\hline
\end{tabular}%
}
\end{table}

\section{Real-World Experiments}
We further evaluate the proposed set-coverage update on two real-world quadruped datasets (Vision60 and Spot) using the learned pseudo-measurement model (see Section~\ref{sec:learned_measurement}). We compare our approach against two baselines: (i) a contact-aided leg kinematics-based \ac{IEKF}~\cite{Hartley2020IJRR}, and (ii) a learned pseudo-measurement model that outputs mean and covariance~\cite{Russell2022TNNLS} with \ac{IEKF} update~\cite{Youm2025ICRA}. Additionally, for the Spot dataset we compare against a \textit{perceptive} radar-based odometry baseline GaRLILEO~\cite{Noh2025arXiv}.

\subsection{Calibration and Coverage Statements}\label{sec:calibration_and_coverage_experiments}
To quantify uncertainty in the pseudo-measurements, we utilize a separate calibration dataset $\mathcal{E}_{\mathrm{cal}} = \{\eout_{t_k}\}_{k=0}^{T_{\text{Seq.}}}$ of measurement errors from the trained network $\netwht$. Since the data-generating process is a dynamical system, samples are temporally dependent. 
To address this complication in practice, we assume the error process is $\beta$-mixing~\cite{McDonald2015EJS}: 
\begin{assumption}[$\beta$-mixing]\label{asm:setup_mixing}
The error process $\{\eout_t\}_{t\geq 0}$ is stationary and $\beta$-mixing with coefficient $\beta(k) \to 0$ as $k\! \to \!\infty$.
\end{assumption}
Assumption~\ref{asm:setup_mixing} allows for approximately independent calibration samples by subsampling at intervals of $K$~\cite{McDonald2015EJS}. This yields the subsampled set
$\mathcal{E}_{\mathrm{cal}}^{\text{Sub.}} = \{\eout_{t_i}\}_{i=0}^{N_{\text{cal}}}$,
where $t_i = i K \Delta t$ and $N_{\text{cal}} = T_{\text{Seq.}}/K$.
We apply split-conformal prediction~\cite{Shafer2008JMLR,Oliveira2024JMLR} using the absolute error in each direction as the score function:
$s_{j,i} = |e_{j,t_i}|$ for $j \in \{x,y,z\}$.
To obtain a \emph{joint} confidence level $\gamma$ for the 3D error vector, we target the per-axis confidence
$\tilde{\gamma} = \gamma^{1/3}$, and set the corresponding significance level $\tilde{\alpha}=1-\tilde{\gamma}$.
Using the standard finite-sample correction, we compute the conformal quantile level based on
$k=\lceil (N_{\text{cal}}+1)(1-\tilde{\alpha})\rceil$~\cite{Shafer2008JMLR,Oliveira2024JMLR},
yielding thresholds $\epsInt = [\epsilon_x,\epsilon_y,\epsilon_z]^\top$ such that
$\PP(|e_{j,t}| \le \epsilon_j) \ge \tilde{\gamma}$ for each $j$.
Assuming independence across coordinates, this implies the joint coverage
$\PP(|\eout_t| \le \epsInt) = \prod_{j\in\{x,y,z\}} \PP(|e_{j,t}| \le \epsilon_j) \ge \tilde{\gamma}^3 = \gamma$.
The resulting constraints on $\eout_t$ define the feasible set $C_t$~\eqref{eq:setup_feasible_set_lie_linear_augmented}
for the coverage-constrained update.

\subsection{Estimation Performance}
\begin{figure}[t]
    \centering
    \includegraphics[width=\linewidth]{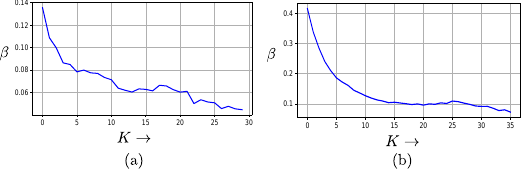}
    \caption{\hspace{-5pt}$\beta$-mixing coefficients for  Vision60 (a) and Spot (b) calibration data.}
    \label{fig:mixing_coeff}
\end{figure}
\begin{figure}
    \centering
    \includegraphics[width=\linewidth]{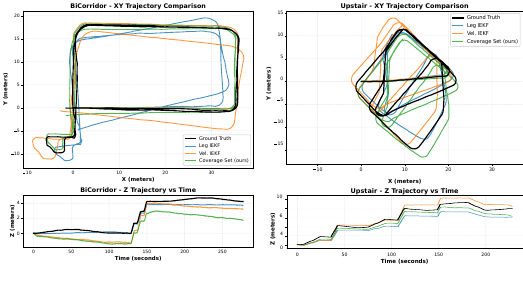}
    \caption{Two of the representative trajectories of the Spot dataset and
with comparison with the baselines.}
    \label{fig:spot_trajectories}
\end{figure}
\begin{figure}
        \centering
        \includegraphics[width=\linewidth]{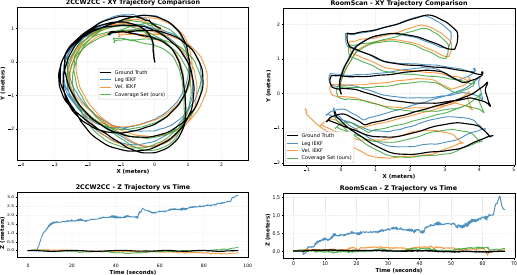}
        \caption{Two of the representative trajectories of the Vision60 dataset and with comparison with the baselines.}
        \label{fig:vision60_trajectories}
\end{figure}

Both real-world datasets pose distinct challenges. Vision60 is recorded in a motion-capture room at $100\,\mathrm{Hz}$ using onboard \ac{IMU} and joint measurements. The trajectories are short and confined to a single room, but include sharp turns and walking over slippery sheets. In contrast, the Spot dataset~\cite{Noh2025arXiv} covers diverse indoor and outdoor terrain with elevation changes. Sequences last multiple minutes and span hundreds of meters, which is particularly challenging for proprioceptive-only odometry.

We set the target coverage level to $\gamma = 0.8$, as we did not observe significant performance difference for $\gamma \in [0.7, 0.85]$ (see Tab.~\ref{tab:unbiased_biased_combined}). From the $\beta$-mixing decay in Fig.~\ref{fig:mixing_coeff}, we choose subsampling intervals $K=20$ for Vision60 and $K=35$ for Spot to approximately satisfy the conformal independence requirement with $\beta\approx 0.05$. This yields calibrated coverage bounds $\epsInt=[0.05,\,0.2,\,0.05]^\top$ on Vision60 using the \textsc{RoomScan} sequence and $\epsInt=[0.22,\,0.05,\,0.1]^\top$ on Spot using the \textsc{Quad} sequence.

Tables~\ref{tab:ghost_dataset_results} and~\ref{tab:spot_dataset_results} report the absolute and relative RMSE for both datasets.  Overall, the proposed set-coverage update remains robust and accurate. On Spot dataset, proprioception-only baselines accumulate substantial drift during sharp turns and abrupt elevation changes, especially in \textsc{Downstair} and \textsc{Overpass}. In these cases, our estimator remains accurate and stable. See Fig.~\ref{fig:spot_trajectories} for representative trajectories. Relative to the \textit{perception-based} GaRLILEO, our \textit{proprioception-only} method achieves similar accuracy and even exceeds it on \textsc{Quad} sequence; GaRLILEO results are taken from~\cite{Noh2025arXiv}. 
Similarly, on the Vision60, the standard Leg \ac{IEKF} exhibits significant vertical drift and high position RMSE (Fig.~\ref{fig:vision60_trajectories}). While learned pseudo-measurements can correct this on flat ground, our method yields the lowest absolute position errors overall amongst the baselines. Overall, by effectively handling non-Gaussian noise, our approach outperforms existing proprioceptive methods while maintaining competitiveness with perception-driven baselines.

\begin{table}[htbp]
\centering
\caption{Results on Spot Dataset}
\label{tab:spot_dataset_results}
\vspace{-8pt}
\scriptsize
\resizebox{\columnwidth}{!}{
\begin{tabular}{l|lccccc}
\toprule
Sequence & Method & APE$_{\text{trans}}$ & RPE$_{\text{trans}}$ & APE$_{\text{rot}}$ & RPE$_{\text{rot}}$ & $\pi<\gamma$ \\
        &        & [m]               & [m]               & [deg]              & [deg/m]            & [\%]         \\
\midrule
\textbf{Atrium}      & GaRLILEO   & 0.816    & 0.055    & 1.715  & 0.554 & --   \\
109.93 m    & Leg IEKF  & 2.267  & 0.108  & 6.457  & 0.908 & --   \\
124.50 s    & Vel. IEKF & 0.453  & 0.037  & 2.566  & 1.494 & --   \\
            & Ours      & 1.094  & 0.055  & 2.075  & 1.746 & 69.7 \\
\midrule
\textbf{BiCorridor}  & GaRLILEO   & 1.425    & 0.063    & 5.519  & 0.885 & --   \\
240.82 m    & Leg IEKF  & 3.251  & 0.141  & 9.130  & 1.413 & --   \\
277.29 s    & Vel. IEKF & 2.236  & 0.110  & 4.118  & 2.436 & --   \\
            & Ours      & 2.526  & 0.074  & 2.658  & 1.955 & 70.0 \\
\midrule
BridgeLoop  & GaRLILEO   & 1.193    & 0.080    & 2.719  & 1.058 & --   \\
161.17 m    & Leg IEKF  & 1.235  & 0.176  & 6.479  & 3.439 & --   \\
187.20 s    & Vel. IEKF & 3.782  & 0.151  & 5.027  & 2.742 & --   \\
            & Ours      & 1.421  & 0.075  & 2.228  & 2.307 & 73.3 \\
\midrule
\textbf{CorriLoop}   & GaRLILEO   & 1.627    & 0.066    & 5.676  & 0.738 & --   \\
208.68 m    & Leg IEKF  & 3.302  & 0.116  & 10.838 & 1.133 & --   \\
229.40 s    & Vel. IEKF & 1.198  & 0.082  & 4.402  & 1.986 & --   \\
            & Ours      & 2.223  & 0.068  & 4.170  & 2.280 & 74.0 \\
\midrule
\textbf{Downstair}   & GaRLILEO   & 3.916    & 0.099    & 3.415  & 1.080 & --   \\
233.75 m    & Leg IEKF  & 8.658  & 0.135  & 9.919  & 1.261 & --   \\
270.90 s    & Vel. IEKF & 45.689 & 0.875  & 3.526  & 2.639 & --   \\
            & Ours      & 6.832  & 0.106  & 9.541  & 2.089 & 70.2 \\
\midrule
\textbf{Overpass}    & GaRLILEO   & 1.526    & 0.091    & 4.043  & 1.227 & --   \\
169.17 m    & Leg IEKF  & 43.267 & 1.522  & 27.234 & 5.870 & --   \\
213.49 s    & Vel. IEKF & 40.923 & 0.846  & 4.870  & 3.367 & --   \\
            & Ours      & 2.897  & 0.054  & 4.866  & 1.763 & 70.6 \\
\midrule
\textbf{Quad}        & GaRLILEO   & 7.347    & 0.080    & 3.356  & 0.838 & --   \\
447.83 m    & Leg IEKF  & 27.624 & 0.111  & 22.735 & 0.902 & --   \\
503.69 s    & Vel. IEKF & 8.648  & 0.083  & 5.375  & 1.889 & --   \\
            & Ours      & 3.180  & 0.065  & 4.186  & 1.826 & 70.0 \\
\midrule
\textbf{SlopeStair}  & GaRLILEO   & 2.359    & 0.061    & 2.805  & 1.051 & --   \\
273.37 m    & Leg IEKF  & 13.382 & 0.118  & 20.625 & 1.203 & --   \\
307.49 s    & Vel. IEKF & 5.221  & 0.113  & 4.339  & 2.583 & --   \\
            & Ours      & 3.058  & 0.079  & 4.559  & 2.046 & 73.7 \\
\midrule
\textbf{Tunnel}      & GaRLILEO   & 3.523    & 0.083    & 2.849  & 0.440 & --   \\
247.94 m    & Leg IEKF  & 8.559  & 0.102  & 9.761  & 0.822 & --   \\
277.00 s    & Vel. IEKF & 5.193  & 0.081  & 5.105  & 1.712 & --   \\
            & Ours      & 5.154  & 0.050  & 3.961  & 1.474 & 70.7 \\
\midrule
\textbf{Upstair}     & GaRLILEO   & 1.496    & 0.071    & 4.048  & 0.933 & --   \\
197.22 m    & Leg IEKF  & 2.005  & 0.171  & 8.068  & 2.885 & --   \\
227.89 s    & Vel. IEKF & 2.132  & 0.151  & 3.161  & 2.776 & --   \\
            & Ours      & 2.042  & 0.072  & 5.037  & 2.441 & 71.2 \\
\midrule
\textbf{Average} & GaRLILEO   & 2.523  & 0.075  & 3.615  & 0.880 & -- \\
                 & Leg IEKF  & 11.355 & 0.270  & 12.775 & 1.984 & -- \\
                 & Vel. IEKF & 11.551 & 0.244  & 4.052  & 2.362 & -- \\
                 & \textbf{Ours} & \textbf{2.736} & \textbf{0.070} & \textbf{4.139} & \textbf{2.003} & \textbf{71.3} \\
\bottomrule
\end{tabular}
}
\vspace{-10pt}
\end{table}

\begin{table}[htbp]
\centering
\caption{Results on Vision60 Dataset}
\label{tab:ghost_dataset_results}
\vspace{-8pt}
\scriptsize
\resizebox{\columnwidth}{!}{
\begin{tabular}{l|lccccc}
\toprule
Sequence & Method & APE$_{\text{trans}}$ & RPE$_{\text{trans}}$ & APE$_{\text{rot}}$ & RPE$_{\text{rot}}$ & $\pi<\gamma$ \\
        &        & [m]               & [m]               & [deg]              & [deg/m]            & [\%]         \\
\midrule
\textbf{2CC2CCW}      & Leg IEKF & 1.344 & 1.419 & 1.622 & 0.280 & -- \\
53.15 m      & Vel. IEKF   & 0.349 & 1.545 & 1.953 & 0.308 & -- \\
76.3 s       & Ours       & 0.278 & 1.506 & 1.225 & 0.299 & 95.3 \\
\midrule
\textbf{2CCW2CC}      & Leg IEKF & 2.095 & 1.355 & 0.893 & 0.290 & -- \\
51.97 m      & Vel. IEKF   & 0.244 & 1.486 & 1.483 & 0.316 & -- \\
95.57 s      & Ours       & 0.226 & 1.492 & 1.119 & 0.315 & 95.7 \\
\midrule
\textbf{3CC3CCW}      & Leg IEKF & 0.756 & 1.375 & 0.940 & 0.282 & -- \\
74.65 m      & Vel. IEKF   & 0.470 & 1.536 & 3.137 & 0.312 & -- \\
107.96 s     & Ours       & 0.667 & 1.548 & 3.083 & 0.312 & 95.3 \\
\midrule
\textbf{3CCW3CC}      & Leg IEKF & 0.901 & 1.382 & 0.869 & 0.287 & -- \\
83.42 m      & Vel. IEKF   & 0.605 & 1.519 & 5.216 & 0.321 & -- \\
129.93 s     & Ours       & 0.363 & 1.520 & 1.987 & 0.315 & 95.1 \\
\midrule
\textbf{RoomScan}     & Leg IEKF & 0.690 & 1.312 & 1.140 & 0.424 & -- \\
38.96 m      & Vel. IEKF   & 0.290 & 1.444 & 1.428 & 0.464 & -- \\
66.97 s      & Ours       & 0.300 & 1.428 & 1.356 & 0.465 & 95.3 \\
\midrule
\textbf{Slipped}      & Leg IEKF & 0.354 & 1.312 & 0.707 & 0.323 & -- \\
19.01 m      & Vel. IEKF   & 0.246 & 1.499 & 1.969 & 0.365 & -- \\
43.03 s      & Ours       & 0.261 & 1.456 & 1.044 & 0.356 & 95.5 \\
\midrule
\textbf{Stroll}       & Leg IEKF & 1.770 & 1.287 & 2.427 & 0.504 & -- \\
99.21 m      & Vel. IEKF   & 0.760 & 1.505 & 2.793 & 0.525 & -- \\
162.3 s      & Ours       & 0.717 & 1.504 & 2.912 & 0.540 & 94.6 \\
\midrule
\textbf{Average} & Leg IEKF & 1.130 & 1.349 & 1.228 & 0.341 & -- \\
                 & Vel. IEKF   & 0.424 & 1.503 & 2.568 & 0.373 & -- \\
                 & \textbf{Ours} & \textbf{0.402} & \textbf{1.493} & \textbf{1.818} & \textbf{0.372} & 95.26 \\
\bottomrule
\end{tabular}
}
\vspace{-10pt}
\end{table}

\section{Conclusion and Future Work}
We proposed a proprioception-only state estimation framework for legged robots that uses a learned measurement model mapping histories of joint-level measurements to body velocity estimates. Under limited data, the error distribution of these predictions is not Gaussian. We systematically characterize this arbitrary error distribution with a set-coverage statement. This set-coverage statement is then used to update the Gaussian state estimate via KL-divergence and moment matching in a computationally efficient way. We compared our method with a baseline in both simulation and a real-world quadrupedal robot dataset after obtaining calibrated set-coverage statements. Our estimator is competitive with both proprioceptive-only and perception baselines in the nominal Gaussian-noise case and remains consistent and robust in the arbitrary-noise case, whereas baselines do not. 
Our future work includes extending the proposed legged state estimator with exteroceptive measurements such as vision.
\appendices

\section{Proof of Theorem \ref{thm:kl_projection_error_state}}
\label{app:proof}

The optimization in \eqref{eq:kl_projection_error_state} is strictly convex in $\bel(\delta\X)$. If $\pi_t \geq \gamma$, the unconstrained minimizer $\belOpt_t = \belPri_t$ (which yields a KL divergence of zero) satisfies all constraints and is therefore uniquely optimal.If $\pi_t < \gamma$, the coverage constraint is active. Introducing dual variables $\lambda > 0$ for the inequality constraint and $\nu$ for the normalization constraint, the KKT conditions necessitate that the first variational derivative of the Lagrangian vanishes:
\begin{equation*}
\frac{\delta \mathcal{L}}{\delta \bel(\delta\X)} = \ln \left( \frac{\belOpt_t(\delta\X)}{\belPri_t(\delta\X)} \right) + 1 - \lambda \Idx_{C_t}(\delta\X) + \nu = 0
\end{equation*}
Solving for $\belOpt_t(\delta\X)$ restricts the optimal solution to the form:
\begin{equation*}
\belOpt_t(\delta\X) = \belPri_t(\delta\X) \exp \left( \lambda \Idx_{C_t}(\delta\X) - \nu - 1 \right)
\end{equation*}
This indicates $\belOpt_t(\delta\X)$ is a piecewise scaling of the prior: $\belOpt_t(\delta\X) = k_1 \belPri_t(\delta\X)$ for $\delta\X \in C_t$, and $\belOpt_t(\delta\X) = k_2 \belPri_t(\delta\X)$ for $\delta\X \in C_t^c$. Enforcing the active coverage constraint $\int_{C_t} \belOpt_t(\delta\X) d\delta\X = \gamma$ using \eqref{eq:pi_def_error_state} directly yields $k_1 \pi_t = \gamma$. Applying the normalization constraint $\int \belOpt_t(\delta\X) d\delta\X = 1$ requires the remaining mass to satisfy $k_2 (1 - \pi_t) = 1 - \gamma$. Solving for the scaling factors $k_1, k_2$ recovers \eqref{eq:posterior_piecewise}. The non-negativity constraint is inherently satisfied by the exponential form.

\bibliographystyle{IEEEtran}
\bibliography{IEEEabrv,references}

\end{document}